\documentclass[pdflatex,sn-mathphys]{sn-jnl}

\jyear{2021}%

\theoremstyle{thmstyleone}%
\newtheorem{theorem}{Theorem}
\newtheorem{proposition}[theorem]{Proposition}%

\usepackage[utf8]{inputenc} 
\usepackage[T1]{fontenc}    
\usepackage{hyperref}       
\usepackage{url}            
\usepackage{booktabs}       
\usepackage{amsfonts}       
\usepackage{nicefrac}       
\usepackage{microtype}      
\usepackage{lipsum}
\usepackage{float}
\floatstyle{plaintop}
\restylefloat{table}
\usepackage[tableposition=top]{caption}
\usepackage{amsmath,amsfonts,amssymb,cancel,siunitx, latexsym}
\theoremstyle{thmstyletwo}%
\newtheorem{remark}{Remark}%
\newtheorem{lemma}{Lemma}
\theoremstyle{thmstylethree}%
\newtheorem{definition}{Definition}%

\renewcommand{\Bbb}{\mathbb}
\begin{document}

\title[A new trigonometric kernel function for support vector machine  ]{A new trigonometric kernel function for support vector machine  }

\author*[1,2]{\fnm{Sajad} \sur{Fathi Hafshejani}}\email{s.fathi@sutech.ac.ir}

\author[1]{\fnm{Zahra} \sur{Moaberfard}}\email{zahra.moaberfard@gmail.com}

\affil*[1]{\orgdiv{Department of Applied Mathematics}, \orgname{Shiraz University of Technology}, \orgaddress{\street{ Modarres Boulevard}, 
\city{Shiraz}, \postcode{71557-13876}, \state{Fars}, \country{Iran}}}
\affil*[2]{\orgdiv{Department of Math and Computer Science}, \orgname{University of Lethbridge}, \orgaddress{\street{ 4401 University Dr W}, \city{Lethbridge}, \postcode{T1K 3M4}, \state{AB}, \country{Canada}}}


\abstract{In the last few years, various types of machine learning algorithms, such as Support Vector Machine (SVM), Support Vector Regression (SVR), and Non-negative Matrix Factorization (NMF) have been introduced. The kernel approach is an effective method for increasing the classification accuracy of machine learning algorithms. This paper introduces a family of one-parameter kernel functions for improving the accuracy of SVM classification. The proposed kernel function consists of a trigonometric term and differs from all existing kernel functions. We show this function is a positive definite kernel function. Finally, we evaluate the SVM method based on the new trigonometric kernel, the Gaussian kernel, the polynomial kernel, and a convex combination of the new kernel function and the Gaussian kernel function on various types of datasets. Empirical results show that the SVM based on the new trigonometric kernel function and the mixed kernel function achieve the best classification accuracy. Moreover, some numerical results of performing the SVR based on the new trigonometric kernel function and the mixed kernel function are presented.}

\keywords{Support vector machine, Kernel-method, Trigonometric kernel function}



\maketitle

\section{Introduction}
Support Vector Machine (SVM) is a supervised learning algorithm mostly used for classification, but it can also be applied for regression. Vapnik \citep{cortes1995support} proposed the SVM method for the first time, and it has been utilized in a wide range of real-world problems such as bioinformatics \citep{byvatov2003support}, biometrics \citep{vatsa2005improving}, power systems \citep{moulin2004support}, and chemoinformatics \citep{doucet2007nonlinear}. In SVM, the training data are used for training and building the classification model. This model is then used to classify unknown samples.

SVM achieves competitive results when the data are linearly separable.  However, kernel functions can be considered {for non-separable data}.  They map the data into a vector space and use linear algebra and geometry to find out the structure of the data. There are some reasons to map the data into a feature space. By mapping the original data into higher-dimensional space, it is possible to transform nonlinear relations within the data into linear ones \citep{tharwat2019parameter}. 

It has been proven that the kernel's theory is based on structural risk minimization by using the maximum margin idea \citep{zhou2022novel}.
The kernel function, which is a key factor in determining algorithm performance, is at the heart of many machine learning algorithms. Kernels allow data to be mapped into {high-dimensional} feature space to increase the computational power of linear machines. Thus, it is a way of extending linear hypotheses to non-linear ones, and this step can be performed implicitly. SVM can be classified into linear and non-linear approaches \citep{liang2020constructing}. Trying to learn a non-linear separating boundary in the input space increases the computational requirements during the optimization phase because the separating surface will be of at least the second order. Instead, SVM maps the data, using predefined kernel functions, into a new but higher-dimensional space, where a linear separator would be able to discriminate between the different classes. The SVM's optimization phase will thus entail learning only a linear discriminant for surfacing the mapped space. Of course, the kernel function's selection and settings are critical for SVM performance. Kernel methods, in fact, transfer the original data into another space, the so-called "feature space," and reveal the data's hidden features \citep{apsemidis2020review}. 

Various kernel functions for machine learning algorithms have been introduced, and some of their significant properties have been explored \citep{smola1999advances,wang2014multi,padierna2018novel}. The Gaussian kernel function is a popular kernel function used in most machine learning algorithms \citep{hoang2019motor}. Table \ref{kernel} demonstrates four popular kernel functions used in most classification papers.

\begin{table}[h!]
\begin{center}

\caption{Four common kernel functions}
 \begin{tabular}{lll}
\hline\noalign{\smallskip}
i& name&$K(x_i,x_j)$ \\
\hline
1&Polynomial kernel&$K_1(x_i,x_j)=(1+x_ix_j)^p$\\
2&Gaussian kernel& $K_2(x_i,x_j)=\exp(-\frac{\|x_i-x_j\|^2}{2\sigma^2})$ \\
3&RBF Kernel & $K_3(x_i,x_j)=\exp(-\gamma{\|x_i-x_j\|^2})$ \\
4& Sigmoid kernel&$K_4(x_i,x_j)=\tanh(\alpha +\beta x_i^Tx_j)$ \\
\noalign{\smallskip}\hline
\end{tabular}
\label{kernel}  
\end{center}
\end{table}
The number of kernel parameters is one of the prime issues in many kernel-based methods. 
Finding the best values for kernel parameters is a challenging problem because their values have a significant impact on their performance. In fact, when the kernel has at least two parameters, the kernel-based method consumes more time in the training phase to find the best values for these parameters. During these years, researchers have been focusing on introducing new kernel functions with one parameter. For example, due to the existence of two parameters in the Sigmoid kernel function in Table \ref{kernel},   machine learning algorithms based on this kernel function spend more time finding the best values for the parameters $\alpha$ and $\beta$. The Gaussian kernel function is an example with one parameter and has been widely utilized in most machine learning algorithms. However, the accuracy of the SVM classification based on this kernel function on some datasets is still not good. In this regard, many efforts have been made to introduce new kernel functions to increase the classification accuracy of the SVM. Recently, Tharwat \cite{parameter2019} suggested a new approach for estimating the best values for kernel parameters. Their method works based on the maximum and minimum distances between samples in each class of dataset.
 
In the two past decades, significant efforts have been devoted to the development of kernel functions for kernel-based methods. Combining kernel functions is a successful approach for kernel-based methods. It has been proven that the combination of two or more kernel functions is also a kernel function, and the new combined kernel function can significantly improve the performance of the SVM \citep{feng2018adaptive}. 

\subsection{Contribution}
\begin{itemize}
    \item This paper gives a new  family of one-parameter trigonometric kernel function to improve accuracy of the SVM and SVR approaches.  We demonstrate that the new kernel function satisfies the kernel conditions. Moreover,  we prove that the new proposed kernel function is a positive definite, which is a critical condition for kernel functions. 
    \item We propose a  new mixed kernel function, which is obtained by combining the new trigonometric function and  the Gaussian function.
    \item We adapt the proposed method in \citep{parameter2019}  to the new trigonometric kernel-SVM and show how to predict the best value for parameter $\sigma$. 
    \item  We perform an extensive implementation to demonstrate the efficiency of the new trigonometric kernel function.  We evaluate the accuracy of the SVM based on the new trigonometric kernel function, the Gaussian kernel function, the polynomial kernel function, as well as a convex combination of the trigonometric kernel function and the Gaussian kernel function on 24 datasets with different ranges of dimensions and numbers of features. Empirical results show that the new proposed kernel function gives good classification accuracy in nearly all the data sets, especially those of high dimensions. It is shown that the new mixed kernel function is very efficient for data classification. Moreover, we examine the efficiency of the SVR method based on the new trigonometric kernel function and the new mixed kernel function.
\end{itemize}
The paper is organized as:
In Section 2, we first recall that some properties of kernel function and then introduce the new trigonometric kernel function. We also investigate some properties of the new kernel function. In section 3, we present a new mixed kernel function by using a convex combination of the new kernel function and the Gaussian kernel function.  Some numerical results of performing kernel-SVM kernel-SVR  on 24 datasets are shown in  Section 3.  
We finally end up the paper by providing some concluding remarks.

We use the following notations conventions throughout the paper:
$\|.\|$ denotes the Euclidean norm of a vector, the non-negative
and positive orthants are denoted by $\Bbb{R}^{n}_{+}$ and $\Bbb{R}^{n}_{++}$ respectively.

\section{A new kernel function}
This section is devoted to presenting a new trigonometric kernel function for SVM and investigating various important properties of the new proposed kernel function. 

We start this section by examining an example. Let us investigate exactly how a kernel function makes data classification easier. Assume $X$ is a randomly generated dataset with two class labels of 1, 0. This dataset has 400 samples and consists of two circles.  We plot the original data in Fig. \ref{beh1} (left). It is obvious that a linear SVM is not able to classify data with a good accuracy. To address this problem, we utilize a kernel function and map data into a new space, Fig. \ref{beh1} (right).  After mapping, a linear SVM can classify data with high accuracy.
\begin{figure}[h!]
	\begin{center}
		\includegraphics[width=.35\textwidth]{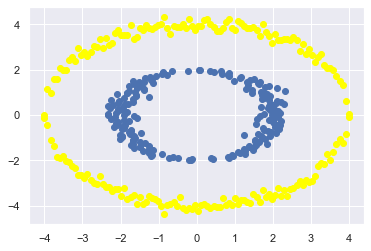}
		\includegraphics[width=.35\textwidth]{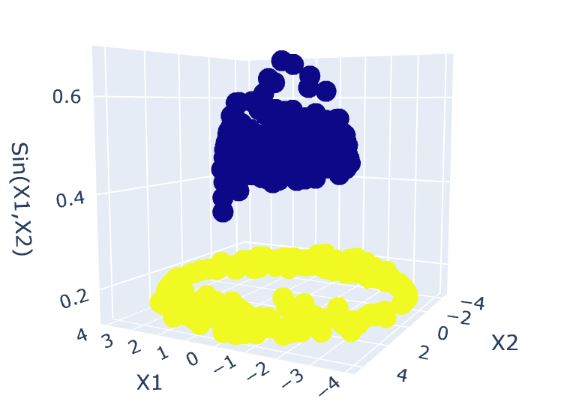}
		\caption{Behavior of kernel function}
		\label{beh1}
	\end{center}
\end{figure}  

Our next goal is to define a new kernel function. To this end, we first remind that the definition of the kernel function demonstrated in  \citep{Berg84}.
\begin{definition}
Let  $X\in \Bbb{R}^d$ be  a non-empty set. A function $K:X\times X\rightarrow \Bbb{R}$ on $X$ is a kernel function if there exists a $\Bbb{K}$-Hilbert space $H$ and a map $\Phi: X \rightarrow H$
such that:
\begin{equation}
K(x,x')=<\Phi(x),\Phi(x')>, \qquad \forall~~x,x'\in X.
\end{equation}
where $\Phi$ is  a nonlinear (or sometimes linear) map from the input space $X$ to the feature space $\mathcal{F}$, and $\langle.,.\rangle$ is an inner product. 
Besides, the function $\Phi$ is so-called a {\it feature map}
and $H$ called a {\it feature space} of $K$.
\end{definition}
\begin{definition}
A kernel function is shift-invariant if it has the following form:
\begin{equation}
    K(x,x')=K(x-x').
\end{equation}
\end{definition}

Now we can introduce a new kernel function for SVM. In this regard, consider the following trigonometric  function:
\begin{eqnarray}\label{ker}
\psi(x)=\sin(h(x)),\qquad
h(x)= \frac{\pi}{2+\sigma x^2},
\end{eqnarray}
where $\sigma$ is a positive real number. The behaviour of the functions $\psi(x)$, $\psi'(x)$, and $h(x)$ with different values of parameter $\sigma$  are  shown in Fig. \ref{fi}.
\begin{figure}[h!]
	\begin{center}
		\includegraphics[width=.35\textwidth]{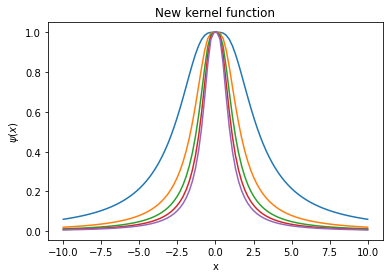}
		\includegraphics[width=.35\textwidth]{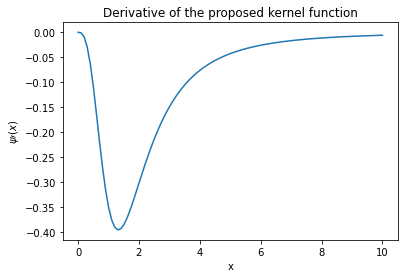}\\
			\includegraphics[width=.35\textwidth]{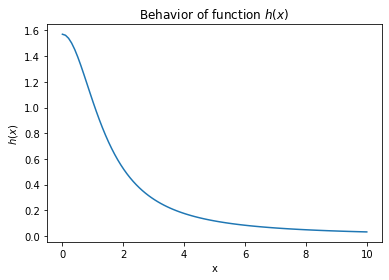}
		\caption{The behaviour of the function $\psi(x)$, $\psi'(x)$, and $h(x)$ with different values of parameter $\sigma$.}
		\label{fi}
	\end{center}
\end{figure} \\
{ 
From Fig. \ref{fi}, we can conclude that:
\begin{itemize}
    \item The new kernel function has a different value for each value of parameter $\sigma$. It implies that, we can figure out the best value for the $\sigma$ for each dataset.
    \item The kernel function's first derivative is negative, implying that $\psi(x)$ is a decreasing function for all $x\geq 0$. 
    \item We have $h(x)\in(0,\frac{\pi}{2}]$, for all $x\geq 0$. 
\end{itemize}}
Next lemma presents some properties of the function given by Eq. \ref{ker}.

\begin{lemma}
Let  $\psi$ be a function defined by Eq. \ref{ker}. Then, we have:
\begin{itemize}
 \item  $\psi(0)=1$.
    \item  $\psi(x)> 0$ for all $x\in \Bbb{R}$.
        \item  $\psi'(x)\leq 0$, for all $x\in \Bbb{R}_+$.  
    \item    $\psi$ is bounded, i.e., for all $x\in \Bbb{R}$, we have $\|\psi(x)\|\leq \psi(0)$.
\end{itemize}
\end{lemma}
\begin{proof}
Using this fact $\sin(\frac{\pi}{2})=1$, we can conclude that the first item is true.\\
To prove the second item,   we know that   $\|.\| $ is non-negative, so $h_1(x)=\frac{\pi}{2+\sigma \|x\|^2}\in
(0,\frac{\pi}{2}]$ for all $x$, which implies  $\sin(h_1(x))\in (0,1]$.
To prove the third item, we have:
\begin{eqnarray}
 \psi'(x)=h'(x)\cos(h_1(t))\nonumber
\end{eqnarray}
It is clear that $h(x)$ is a decreasing function, so we can conclude that $h'(x)<0$. On the other hand, $\cos(h(x))$ is a positive function for all $x\in (0,\frac{\pi}{2}]$, i.e., the first derivative of the function $\psi(x)$ given by Eq. \ref{ker} is negative. 

The fourth item of the lemma is obtained by using the fact that $\psi(0)=1$,
$\psi(x)\leq 0$ and $ \psi (x)\geq0$ for all $x\geq0$.
\end{proof}

Now, we are in the position to define a new trigonometric kernel function for the SVM approach for the first time.
\begin{definition}
Suppose $x,x'\in X$, and the feature map function $\psi(x)$ given by Eq. \ref{ker}, then a new trigonometric kernel function is defined by:
\begin{equation}\label{kernel_func}
    K(x,x')=\sin(\frac{\pi}{2+\sigma\|x-x'\|^2})
\end{equation}
where $\sigma$ is a  positive real constant. 
\end{definition}
Using the symmetric property for the inner product, the first property of the new proposed kernel function can be expressed as follows:
\begin{equation}
K(x,x')=\sin(\frac{\pi}{2+\sigma\|x-x'\|^2})=\sin(\frac{\pi}{2+\sigma\|x'-x\|^2})=K(x',x),
\end{equation}
It implies that the new trigonometric kernel function satisfies the symmetric property.

The following definition recalls another property of the kernel function called {positive definite}. 
\begin{definition} (Positive definite kernel \cite{Berg84}) Let $X $ be a nonempty set. $K: X \times X \rightarrow  \Bbb{R}$ is called a positive kernel function if and only if for any $c\in\Bbb{R}^m$, the following inequality is true:
\begin{equation}
  \sum_{i,j=1}^{m}c_ic_jK_{i,j}\geq 0,
\end{equation}
in which  $K_{i,j}=K(x_i,x_j)$.
\end{definition}
\begin{remark}\label{t2}
Suppose that $K: X \times X \rightarrow  \Bbb{R}$  is a kernel function. Then $K$ is-so
called strictly positive definite if and only if for any $c\in\Bbb{R}^m$, we have:
\begin{equation}
  \sum_{i,j=1}^{m}c_ic_jK_{i,j}> 0.
\end{equation}
\end{remark}
Chasing the mentioned conditions is very challenging for some kernel functions.  A common way to address these challenges is to utilize the matrix of the kernel function. Van Den Berg et al.  \cite{Berg84} proved that a kernel function is positive definite if and only if the symmetric matrix $K$ is positive definite. Note that, components of the matrix $K_{i,j}$ are $K(x_i,x_ j)$, as well as $K_{i,i}=K(x_i,x_i)=1$ for $i=1,2,...,n$. It shows all elements on the main diagonal matrix $K$ are one.
Now, we are in the position to  present the positive definite property for the new proposed kernel function, which delineates some properties of the new proposed kernel function.
\begin{lemma}[\citep{Berg84}]\label{t1}
Suppose  $K: X \times X\rightarrow\Bbb{R}$ is a kernel function. Then $K$ is positive
definite if and only if
\begin{equation}
    \det(K(x_i,x_j)_{i,j\leq n})\geq0,\qquad ~ \{x_1,x_2,...,x_n\}\subseteq X.\nonumber
\end{equation}
\end{lemma}
\begin{theorem}\label{positive_matrix} The new proposed kernel functions defined by Eq. \ref{ker}  is positive definite.
\end{theorem}
\begin{proof}
We begin  by induction on $n$.
Considering the structure of matrix $K$ as:
\begin{equation}\label{matrixk}
K = \left(
\begin{array}{cccc}
1& l_{12}&\cdots&l_{1n} \\
l_{21} &1&\cdots&l_{2n}\\
\vdots&\vdots&\cdots&\vdots\\
l_{n1}&l_{2n}&\cdots&1
\end{array} \right),
\end{equation}
where $0\leq l_{ij}<1,~~i\neq j, 1\leq i,j\leq n$, and $l_{ij}=K(x_i,x_j) $. 
The property  $K(x_i,x_j)=K(x_j,x_i)$ implies $K$ 
is a symmetric matrix. For the matrix $K$, we have:
\begin{eqnarray*}
n=1,&&\qquad det(K_{1,1})=1>0;\\
n=2,&&\qquad det(K_{2,2})=1-l_{12}^2>0,
\end{eqnarray*}
in which $K_{i,i}$ denotes a  square matrix with dimension $i$. Let us suppose the theorem be  true for $n-1$. We
prove that $det(K_{n,n})>0$.  Using the fact that $K_{11}=1>0$ and
by subtracting $K_{1,1}$  times the first column from the
$K$-th column, $K = 2,\ldots,n$,  the new matrix where the first column remained unchanged and other columns are changed by using the relation $K'_{jk} = K_{jk}-K_{1k}K_{j1}$ for $k\geq 2$ can be obtained.
In addition, the new matrix has the same principal minors as the matrix $K$, so we have:
\begin{equation}
    det(K'_{j,k\leq p})=det(K_{j,k\leq p}),\quad  p=1,2,\ldots,n.
\end{equation}
So, the matrix $K'$ can be rewritten as:
\begin{equation}\label{matrixk'}
K' = \left(
\begin{array}{cccc}
1& 0&\cdots&0 \\
0 &k'_{22}&\cdots&k'_{2n}\\
\vdots&\vdots&\cdots&\vdots\\
0&k'_{2n}&\cdots&k'_{nn}
\end{array} \right),
\end{equation}
Determinant of the matrix $K$ can be calculated as follows:
\begin{eqnarray*}
det(K)=1\ast det\left(
\begin{array}{ccc}
k'_{2,2}&\cdots&k'_{2,n}\\
\vdots&\cdots&\vdots\\
k'_{2,n}&\cdots&k'_{n,n}
\end{array} \right)>0,
\end{eqnarray*}
where the inequality is obtained from the fact that determinate of matrix $K_{n-1,n-1}$ is positive.
\end{proof}
Theorem \ref{positive_matrix} implies that the proposed  kernel functions is positive definite.

\section{A new mixed kernel function}
Mixing or combining several kernel functions into one kernel function was used to improve
the performance of a single kernel function
 \citep{kung2014kernel}.  Here, we utilize this idea and suggest a new mixed kernel function, that is a convex combination of the Gaussian kernel function and the new trigonometric kernel function given by Eq. \ref{kernel}. For this purpose, we first recall a lemma allowing us to combine two kernel functions.
\begin{lemma}\label{comba}
Let $f$ and $g$ be two  kernel functions. Then, a convex combination of $f$ and $g$, i.e.,  $\beta f+(1-\beta)g$ is also a kernel function for all $\beta\in[0,1]$. 
\end{lemma}
Using Lemma \ref{comba}, we introduce a new  mixed kernel function.
\begin{proposition}
If $\beta \in [0,1]$, then the new  mixed kernel function is given by:
\begin{equation}\label{c_ker}
K_c(x,x'):=\beta\sin(h_1(x,x'))+(1-\beta)e^{\frac{\|x-x'\|^2}{2\sigma^2}},\quad  \texttt{where}\quad h_1(x,x')=\frac{\pi}{2+\sigma\|x-x'\|^2}.
\end{equation}   
\end{proposition}
Using Lemma \ref{comba}, we have the following Remark. 
\begin{remark}
The kernel function defined by Eq. \ref{c_ker} is a positive definite kernel function. 
\end{remark}
Note that Eq. \ref{c_ker} is a kernel function and satisfies the properties of the kernel function. To compare behaviour of the new  mixed kernel with the Gaussian kernel function and new trigonometric kernel function, we  plot behaviour  of these functions in {Fig \ref{beh}. This figure clearly shows that the value of the mixed fraction is always between the new proposed kernel function and the Gaussian function. In fact, the mixed kernel function can control the effects of changing the parameter $\sigma$ well.}

\begin{figure}[h!]
	\begin{center}
		\includegraphics[width=.45\textwidth]{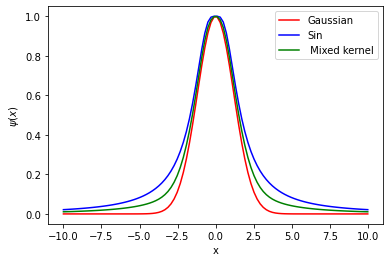}
		\caption{Behavior of the new  mixed kernel function, the Gaussian kernel function and new trigonometric kernel function.}
		\label{beh}
	\end{center}
\end{figure}  

\section{Numerical results}
This section gives some numerical results of  performing the kernel-SVM and kernel-SVR on various data sets. 
\subsection{SVM problem}
The kernel SVM optimization problem can be express as:
\begin{eqnarray*}  &&~\min_{\alpha}~~\sum_{i=1}^n\sum_{j=1}^{n}\alpha_i\alpha_jy_iy_jK(x_i,x_j)-\sum_{i=1}^{n}\alpha_j\\
   && ~~~S.t.~~ \sum_{i=1}^ny_i\alpha_i=0,\quad 0\leq \alpha_i\leq C,
\end{eqnarray*}
where $n$ denotes the number of samples, $C$ is the penalty or
regularization parameter, and  $K(x_i,x_j)$ is a kernel function. 
\subsection{Kernel functions}
We implement the SVM  based on the following kernel functions:
\begin{eqnarray}
    K_1(x,x_j)&=&(1+xx_j)^p,\quad p=2\label{pol}\\
    K_2(x,x_j)&=&\exp(-\frac{\|x-x_j\|^2}{2\sigma^2})\label{gus}\\
    K_3(x,x')&=&\sin(\frac{\pi}{2+\sigma\|x-x'\|^2})
    \label{new_ker}\\
    K_4(x,x')&=&\beta\sin(\frac{\pi}{2+\sigma\|x-x'\|^2})+(1-\beta)e^{\frac{\|x-x'\|^2}{2\sigma^2}},\quad \beta =\frac{1}{2}.\label{multi}
\end{eqnarray}

\subsection{Parameter selection}
We use the idea presented in \citep{parameter2019} to find the best values for the parameters $C$ and $sigma$. In fact, we select the best values for $sigma$ and $C$ from the following set: 
\begin{equation*}
    c,~ \sigma \in\{2^{-5},2^{-4},...,2^{10}\}.
\end{equation*}
To apply  \citep{parameter2019}'s idea, we first calculate the minimum and maximum distances between samples in each class. After that, we choose large values for $\sigma$ when the dataset is compact and small values when the dataset is sparse. In this regard, we consider the following three datasets.
\begin{itemize}
    \item {\bf{Gaussian dataset}}: Gaussian  dataset has 400 samples and two classes. Each sample has two features.  This is a compact dataset because the maximum distances between samples of the first and second classes are 4.736 and 5.8477, respectively, and the minimum distances between samples of the first and second classes are 0.0412 and 0.1427, respectively. We run the kernel-SVM with different values of the parameters $\sigma$ and $C$. The obtained results for $C=1$ and $C=10$ are shown in Tables \ref{Guass1} and \ref{Guass2}. Large values for $\sigma$ produce better results because the distance between samples is small. 

\begin{table}[H] 
\centering
\begin{tabular}{|c|c|c|c|c|c|c|c|c|c|c|}
\hline
Results&$\sigma=0.1$&$\sigma=1$&$\sigma=2$&$\sigma=10$&$\sigma=50$&$\sigma=100$&$\sigma=1000$\\
\hline
\# SV&2&156&315&128&173&215&315\\
\# TrE.&63&48&43&27&17&10&3\\
\# TsE.&17&35&35&27&22&18&14\\
\hline
\end{tabular}
    \caption{The number of misclassified training samples (\# TrE.), number of misclassified testing samples (\# TsE.), and number of support vectors (\# SVs) of the new trigonometric  kernel SVM using different values of parameter $\sigma$ with $C=1$  with Gaussian dataset. }
    \label{Guass1}
\end{table}

\begin{table}[H] 
\centering
\begin{tabular}{|c|c|c|c|c|c|c|c|c|c|c|}
\hline
Results&$\sigma=0.1$&$\sigma=1$&$\sigma=2$&$\sigma=10$&$\sigma=50$&$\sigma=100$&$\sigma=1000$\\
\hline
\# SV&267&180&199&93&317&179&312\\
\# TrE.&108&66&62&31&14&3&0\\
\# TsE.&31&24&22&33&35&19&16\\
\hline
\end{tabular}
    \caption{The number of misclassified training samples (\# TrE.), number of misclassified testing samples (\# TsE.), and number of support vectors (\# SVs) of the new trigonometric  kernel SVM using different values of parameter $\sigma$ with $C=10$  with Gaussian dataset.}
    \label{Guass2}
\end{table}

\item {\bf{College dataset}}: The college dataset has 777 samples and two classes. Each sample has 17 features.  The first class has 565 samples, and the second has 212 samples.  We discovered that the maximum distances between samples of the first and second classes are $20,111$ and $47,861$, respectively, and the minimum distances between samples of the first and second classes are 0. Small values for $\sigma$ can produce better results in this case. We run the kernel-SVM with various values of the $sigma$ and $C$ parameters. The obtained results for $C=1$ and $C=10$ are shown in Tables \ref{colleg1} and \ref{colleg10}.
    
\begin{table}[H] 
\centering
\begin{tabular}{|c|c|c|c|c|c|c|c|c|c|c|}
\hline
Results&$\sigma=0.1$&$\sigma=1$&$\sigma=2$&$\sigma=10$&$\sigma=50$&$\sigma=100$&$\sigma=1000$\\
\hline
\# SV&621&621&621&621&621&621&621\\
\# TrE.&0&0&0&0&0&0&0\\
\# TsE.&33&29&30&32&31&33&39\\
\hline
\end{tabular}
    \caption{The number of misclassified training samples (\# TrE.), number of misclassified testing samples (\# TsE.), and number of support vectors (\# SVs) of the new trigonometric  kernel SVM using different values of parameter $\sigma$ with $C=1$  with College dataset.}
    \label{colleg1}
\end{table}

\begin{table}[H] 
\centering
\begin{tabular}{|c|c|c|c|c|c|c|c|c|c|c|}
\hline
Results&$\sigma=0.1$&$\sigma=1$&$\sigma=2$&$\sigma=10$&$\sigma=50$&$\sigma=100$&$\sigma=1000$\\
\hline
\# SV&621&621&621&621&621&621&621\\
\# TrE.&0&0&0&0&0&0&0\\
\# TsE.&31&31&28&228&34&34&38\\
\hline
\end{tabular}
    \caption{The number of misclassified training samples (\# TrE.), number of misclassified testing samples (\# TsE.), and number of support vectors (\# SVs) of the new trigonometric  kernel SVM using different values of parameter $\sigma$ with $C=10$  with College dataset}
    \label{colleg10}
\end{table}
\item  {\bf{Tic-Tac-Toc dataset:}}  This dataset has two classes: the first class has 323 samples, and the second has 626 samples and nine features. The maximum distances between samples of the first and second classes were 4.90 and 4.47, respectively. In this case, we need large values for parameter $\sigma$. We present a summary of the results in Table \ref{tic}.
\begin{table}[H] 
\centering
\begin{tabular}{|c|c|c|c|c|c|c|c|c|c|c|}
\hline
Results&$\sigma=0.1$&$\sigma=1$&$\sigma=2$&$\sigma=10$&$\sigma=50$&$\sigma=100$&$\sigma=1000$\\
\hline
\# SV&674&348&476&751&766&766&766\\
\# TrE.&37&8&0&0&0&0&0\\
\# TsE.&47&35&28&24&19&16&22\\
\hline
\end{tabular}
    \caption{The number of misclassified training samples (\# TrE.), number of misclassified testing samples (\# TsE.), and number of support vectors (\# SVs) of the new trigonometric  kernel SVM using different values of parameter $\sigma$ with $C=1$  with the Tic-Tac-Toc dataset}
    \label{tic}
\end{table}
\end{itemize}

\subsection{Datasets}
We perform the SVM on the 24 datasets with different number of samples and features\footnote{https://www.openml.org/search?q=gisette\&type=data\&sort=runs\&status=active}.
 There are two
data samples are used in all experiments; $80\%$ of the data were used to train SVM and $20\%$ of the data were used to test model.
Table \ref{data1} gives details about 24  datasets used in this section. Note that ``Name'' denotes the name of the dataset, ``F'' presents the number of the features of the dataset, and ``N'' is the number of samples.

\begin{table}[H] 
\centering
\begin{tabular}{|l|lll|l|lllllll|}
\hline
i&Name&\#F&\#N&i&Name&\#F&\#N\\
\hline
1&Banana&2&5300&13&Sonar&60&208\\
2&Bill-authentication&4&1372&14&SPECT&22&80\\
3&Carseats&10&400&15&Weekly&8&1089\\
4&College&17&777&16&Wisc-bc-data&31&569\\
5&Column-2C-weka1&6&310&17&PhpVDlhKL&230&64\\
6&Dataset-spine&12&309&18&Fri-c1-100-10&11&100\\
7&Ex2data1&2&100&19&Svmguide1&4&3089\\
8&Ex2data2&2&118&20&tic-tac-toe.data&10&958\\
9&GaussianData&2&400&21&analcatdata$_-$lawsuit&5&264\\
10&GSE58606-data1&1927&133&22&phpAmSP4g&31&569\\
11&Numeric-sequence&28&2400&23&kc2&22&5228\\
12&Pimalndians1&8&390&24&haberman&4&306\\
\hline
\end{tabular}
    \caption{Data Sets}
    \label{data1}
\end{table}

\subsection{Results of performin SVM }
Classification accuracy for SVM  based on four  kernel functions given by Eqs. \ref{pol}--\ref{multi} are demonstrated  in Fig \ref{all_re}. Note that x-axis denotes the number of dataset (i.e., $i$) in Table \ref{data1}. In addition,
``$K_1$" shows the accuracy results for  polynomial kernel-SVM, ``$K_2$'' denotes the results for the  Gaussian kernel function, and ``$K_3$'' and ``$K_4$'' are results for the new trigonometric  kernel function and  mixed kernel function given by Eq. \ref{multi}, respectively.

\begin{figure}[h!]
	\begin{center}		\includegraphics[width=0.75\textwidth]{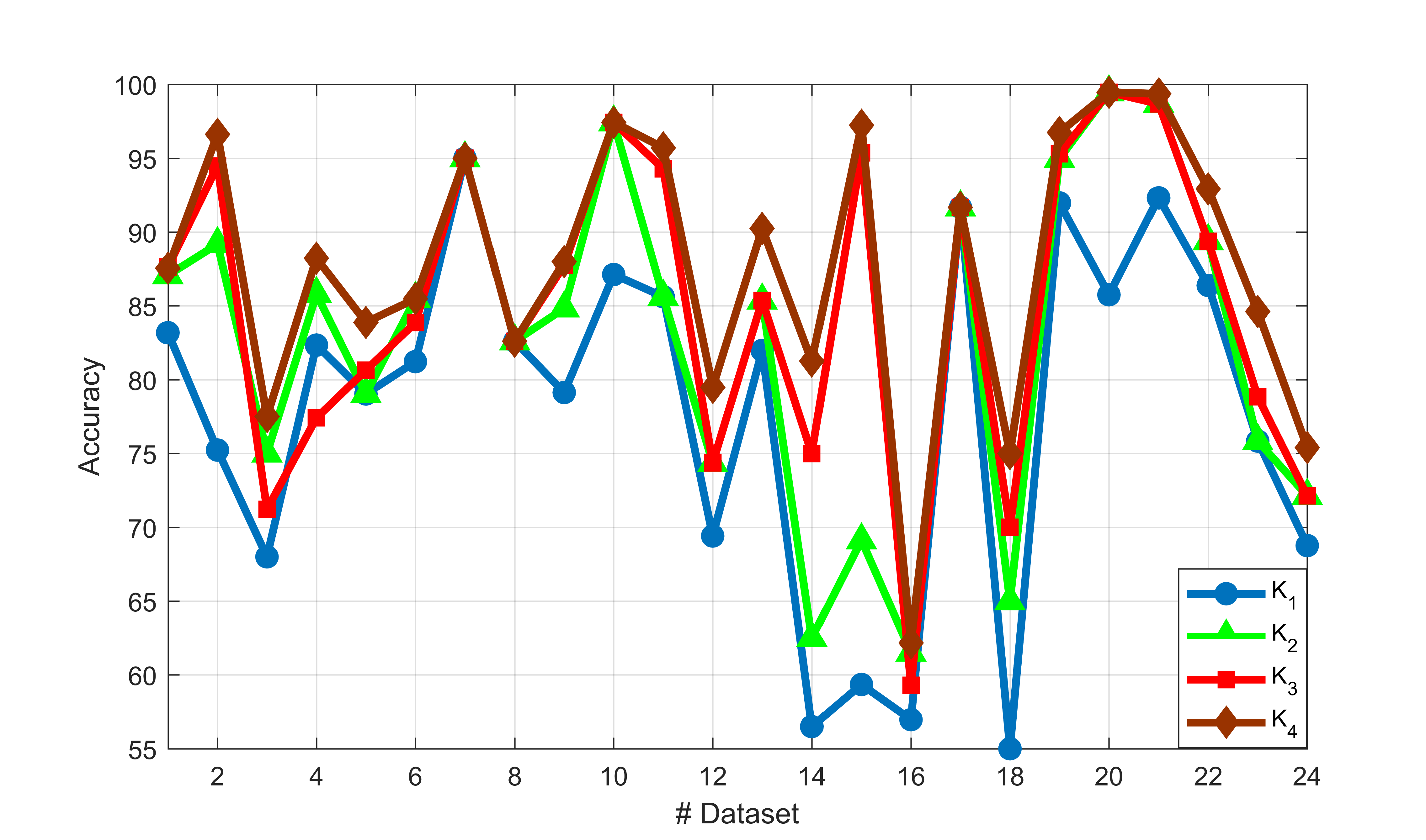}
		\caption{Accuracy results of performing SVM based on the four kernel functions on 24 data-set denoted in Table 7.}
		\label{all_re}
	\end{center}
\end{figure}


\subsection{Kernel SVR}
Here we investigate efficiency of the proposed kernel functions for SVR approach. In this regard, we implement the SVR based on the  kernel functions given by Eqs. \ref{new_ker}  and \ref{multi} on a set of 200 data points generated by using the following function:
\begin{equation*}\label{fun_svr}
f(x)=\sin(x)*\exp(-0.2*x)+\epsilon
\end{equation*}
where $\epsilon$ is a random number. \\
{Fig. \ref{svr}}  demonstrates the behaviour of the SVR based on two mentioned kernel functions.  
 { This figure shows SVR based on the new proposed kernel function and mixed kernel function can fit the best curve within an $\epsilon$ tube to predict the model.  }
\begin{figure}[h!]
	\begin{center}
\includegraphics[width=0.75\textwidth]{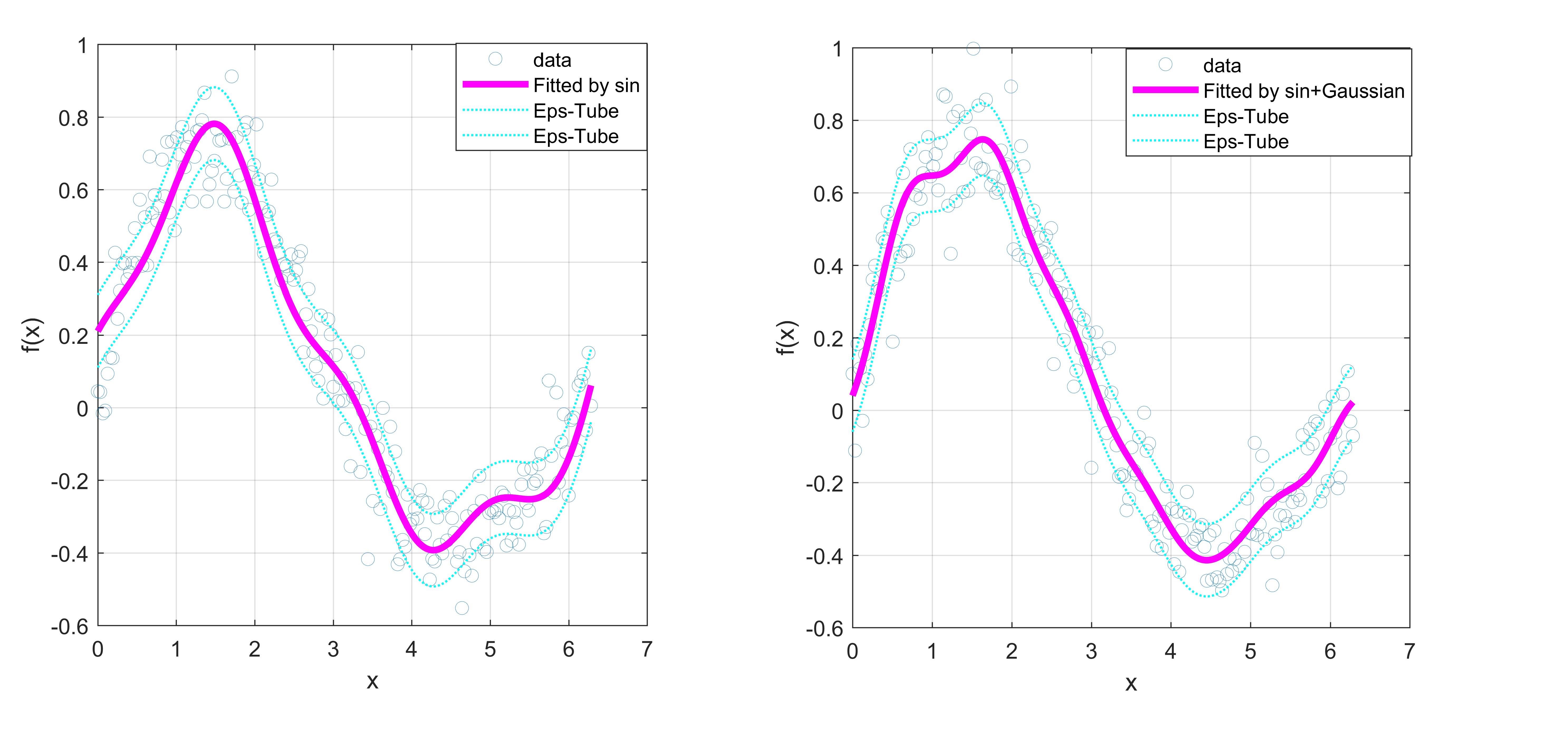}
		\caption{Behaviour of SVR based on the kernel functions given by Eq. (\ref{new_ker}) and Eq. (\ref{multi}). }
		\label{svr}
	\end{center}
\end{figure}

\subsection{Discussion}
Based on  Tables \ref{Guass1}-- \ref{data1}, Figs. \ref{all_re} and \ref{svr}, we can conclude that:
\begin{itemize}
\item A variety of datasets with different numbers of samples and features have been used.
\item  The average classification accuracies of the SVM method based on the polynomial, Gaussian, new trigonometric, and mixed kernel functions are  $77.93\%$, $82.79\%$, $84.88\%$, and $87.64\%$ respectively.
    \item The new trigonometric kernel-SVM  improved the accuracy of the Gaussian kernel-SVM (from $82.79\%$ to $84.88\%$) based on average accuracy.  
    \item The new  mixed kernel function improved the classification accuracy of the SVM. According to the results, it improves the $4.85\%$ accuracy of the Gaussian kernel-SVM and the $2.66\%$ accuracy of the new proposed kernel-SVM. 
    \item The new kernel-SVM achieves better classification accuracy for 10 datasets than the Gaussian kernel-SVM.
    \item The new proposed kernel function achieves the best results with small values of $\sigma$ when the dataset is sparse and with large values of $\sigma$ when the dataset is compact.
    \item The SVR based on the new proposed kernel functions is able to predict an appropriate graph with the minimum error.
\end{itemize}

\section{Concluding remarks}
In this paper, we introduced a new trigonometric kernel function for machine learning. Various properties of the proposed kernel function were investigated. Then we combined the proposed kernel function with the Gaussian kernel function and introduced a new mixed kernel function. The numerical results confirmed that the new proposed trigonometric kernel function improved the classification accuracy of the SVM for the most considered datasets. Moreover, the mixed kernel function achieved the best results in terms of classification accuracy. 

\section*{Declarations}

\subsection*{Acknowledgements} The authors would like to thank the editors and anonymous reviewers for their constructive comments.

\subsection*{Funding} No funding was received for conducting this study.
\subsection*{Ethics approval and consent to participate} Not applicable. 
\subsection*{Consent for publication}
Not applicable.
\subsection*{Availability of data and material} The dataset used in this paper is available from the following link \textcolor{blue}{ https://paperswithcode.com/dataset/orl}

\subsection*{Conflict of Interests} The authors have no conflicts of interest to declare that are relevant to the content of
this article.

\subsection*{Competing interests} The authors have declared that no competing
interests exist.

\bibliography{sn-bibliography}


\end{document}